\documentclass[jair_clean,twoside,11pt,theapa]{article}
\usepackage{jair_clean, theapa, rawfonts}

\usepackage{url}

\usepackage{bbm}
\usepackage{amsfonts,amsmath,amssymb,amsthm}

\newcommand{\bbbr}{\ensuremath \mathbb{R}}
\newcommand{\Neigh}{\ensuremath \mathcal{N}}
\newcommand{\citep}[1]{\cite{#1}}
\newcommand{\citet}[1]{\citeA{#1}}

\newtheorem{definition}{Definition}
\newtheorem{theorem}{Theorem}
\newtheorem{lemma}{Lemma}
\newtheorem{claim}{Claim}
\newtheorem{corollary}{Corollary}

\jairheading{?}{201?}{??-??}{?/14}{?/??}
\ShortHeadings{Sparse Discretization for Graphical Games}
{Ortiz}
\firstpageno{1}

\begin{document}

\title{On Sparse Discretization for Graphical Games}

\author{\name Luis E. Ortiz \email leortiz@cs.stonybrook.edu \\
       \addr Department of Computer Science, Stony Brook University,\\
       Stony Brook, NY  11794-4400 USA}


\maketitle

\begin{abstract}
This short paper concerns discretization schemes for representing and computing approximate Nash equilibria, with emphasis on graphical games, but briefly touching on normal-form and poly-matrix games. The main technical contribution is a representation theorem that informally states that to account for every exact Nash equilibrium using a nearby approximate Nash equilibrium on a grid over mixed strategies, a uniform discretization size linear on the inverse of the approximation quality and natural game-representation parameters suffices. For graphical games, under natural conditions, the discretization is logarithmic in the game-representation size, a substantial improvement over the linear dependency previously required. The paper has five other objectives: (1) given the venue, to highlight the important, but often ignored, role that work on constraint networks in AI has in simplifying the derivation and analysis of algorithms for computing approximate Nash equilibria; (2) to summarize the state-of-the-art on computing approximate Nash equilibria, with emphasis on relevance to graphical games; (3) to help clarify the distinction between sparse-discretization and sparse-support techniques; (4) to illustrate and advocate for the deliberate mathematical simplicity of the formal proof of the representation theorem; and (5) to list and discuss important open problems, emphasizing graphical-game generalizations, which the AI community is most suitable to solve.
\end{abstract}

\section{Introduction}
There has been quite a bit of work over last 10 years, mostly, but not
exclusively, in the theoretical computer
science community, on the problem of computing
approximate Nash equilibria in games (see, e.g., the work of
\citeR{Kearns_et_al_2001,Vickrey_and_Koller_2002,ortizandkearns03,Lipton:2003:PLG:779928.779933,Singh_et_al_2004,Soni_et_al_2007,Daskalakis_et_al_2007,TsaknakisOptANE,FederANE,Daskalakis_and_Papadimitriou_2008,Daskalakis_et_al_2009,Kontogianis_et_al_2009,HemonANE,AwasthiStable,BosseNew,PonsenANE}
and the references therein; see also the book by~\citeR{nisan07} for
additional references).~\footnote{To keep the focus and length of this short paper, the presentation and motivation of the significance and broader
scientific and practical impact the study of
Nash equilibria is beyond the scope of this paper. Similarly, the main focus here is on
\emph{graphical games (GGs)} and \emph{absolute approximations} of
Nash equilibria,
the most common type of approximation found in the
literature. However, the curious reader can go to
Appendix~\ref{app:scope} to find a brief discussion of the significance of
Nash equilibria (Appendix~\ref{app:NE}), other classes of graphical models for game
theory (Appendix~\ref{app:GMs}), and other types of approximations of
Nash equilibria (Appendix~\ref{app:rel}).}

This short paper revisits the simple \emph{uniform-discretization scheme} that~\citet{Kearns_et_al_2001} originally introduced
in the context of $n$-player $2$-action \emph{graphical games
  (GGs)}. 
 (Formal definitions of concepts
from game theory, graph theory, AI, and uniform-discretization appear in
Section~\ref{sec:prelim}. The presentation in the Introduction will
remain informal.)

\citet{Kearns_et_al_2001} 
showed that if the size of the individual grid that the uniform-discretization scheme induces over the probability of playing an
action
was
$O(2^k/\epsilon)$, where $k$ is the size of the largest set of
neighbors of any player and $\epsilon$ is an approximation quality parameter,
then for each exact Nash equilibrium of the GG, its closest joint
mixed-strategy in the resulting regular \emph{joint} grid is an approximate
$\epsilon$-Nash equilibrium of the GG.~\footnote{In an
$\epsilon$-Nash equilibrium, players tolerate losing expected gains, up to an $\epsilon$
amount, from not unilaterally deviating.} 
~\citet{Kearns_et_al_2001}  used that discretization
to design a special type of dynamic-programming algorithm tailored to
computing approximate Nash equilibria in GGs with tree-structured graphs they called {\bf TreeNash} 
that runs in time linear in
the number of players and $O(2^{k^2})$, assuming a fixed $\epsilon$. The size of
the input representation of the $n$-player $2$-action GG is $O(n\, 2^k)$. In
collaboration with Prof. Kearns, we later extended {\bf TreeNash} as a heuristic for
GGs with loopy graphs, leading to an algorithm we called
{\bf NashProp}, which stands for ``Nash Propagation''~\citep{ortizandkearns03}.

An unpublished note drafted back
in December 2002~\citep{Ortiz_2002}, later posted online\footnote{\url{http://www.cis.upenn.
edu/~mkearns/teaching/cgt/revised_approx_bnd.pdf}} as part of a course
on computational game theory
taught by Prof. Michael Kearns during the Spring 2003 at the University of
Pennsylvania, provided a
significantly sharper bound of $O(k/\epsilon)$ on the size of the discretization required
to achieve the same approximation result. The revised bound was \emph{logarithmic} in the
representation size of the game, as opposed to the previous
\emph{linear} bound that~\citet{Kearns_et_al_2001} derived.~\footnote{Note that~\citet{Kearns_et_al_2001} only considered the
case of \emph{binary} actions, i.e., $m=2$, hence $m$ does not play a
role in the results within the context of that paper.} The revised,
significantly tighter upper-bound led to an improved running time of
$O(k^k)$ for {\bf TreeNash} in terms of just $k$, which meant that,
when using
the sparser discretization derived in the old note~\citep{Ortiz_2002}, {\bf TreeNash} becomes a quasi-polynomial
time approximation scheme (quasi-PTAS) to compute an
$\epsilon$-Nash equilibrium.
~\citet{Daskalakis_and_Papadimitriou_2008} independently
  rediscovered this result,
about five years later, 
using a
  considerably more complex approach to the proof/derivation than the
  simple, algebraic
  approach presented in the old note~\citep{Ortiz_2002} and here.

Similarly, some of the results regarding algorithmic implications presented here that followed from the
old note~\citep{Ortiz_2002}, particularly for
normal-form games, have also been independently
rediscovered in the
literature using different approaches throughout the years (see, e.g.,
some of the results of~\citeR{Lipton:2003:PLG:779928.779933},
and~\citeR{Daskalakis_and_Papadimitriou_2008}).
Those particular results, discussed in more detail in the technical sections
of this paper, followed immediately from the
improved discretization bound given in the old note~\citep{Ortiz_2002}, combined with previously known results
from the CSP and graphical-models literature in AI.

The present short paper extends the old note~\citep{Ortiz_2002} and
shows how the improved discretization-size bounds for GGs,
and some specializations, fall off immediately as corollaries of
a theorem that holds for GG generalizations.

In particular, the current paper presents a result on the sufficient size for uniform
discretization to capture, in a formal sense, every possible Nash
equilibrium of the game via another close approximate Nash equilibria
in the induced grid over the space of mixed strategies of players in
the game. As example corollaries of the main result, for graphical
games with largest neighborhood size $k$, the sufficient size is $O(k
m / \epsilon)$, which implies, $O(n m/\epsilon)$ for standard
normal-form games, while for $n$-player $m$-action poly-matrix games, the sufficient size is $O(m / \epsilon)$. 

Similarly, the representation result yields immediate computational results based on connection to algorithms for {\em constraint satisfaction problems (CSPs)}. Section~\ref{sec:alg} presents and discusses several results on polynomial and quasi-polynomial time algorithms for several interesting subclasses of graphical and normal-form games that fall off from that connection.

The paper ends with a discussion of the algorithmic implications that the main
technical result
may have on other work in computational game theory. It also lists
several open problems.


\section{Motivation}
\label{sec:motiv}

This section provides additional motivation for the necessity and significance of
the study of the problem of computing approximate Nash equilibria, particularly in
GGs, and for the
development and dissemination of this short paper, with emphasis on
its particular relevance to the AI community.

\subsection{The Computational Complexity of Nash Equilibria: State of Affairs}
\citet{chenanddeng05_two} settled the complexity of computing an
exact Nash equilibrium in bimatrix games,
a 50-years-old
open problem, proving the
problem to be 
PPAD-complete (see also later work by~\citeR{chenanddeng06} and~\citeR{ChenSettlingJACM}).~\footnote{\citet{PapadimitriouPPAD} introduced
the complexity class PPAD to characterize fixed-point-type problems
such as those of
Nash equilibria in which a solution always exists. A discussion on the complexity
class PPAD is beyond the scope of this short paper.} ~\citet{DP:complexityofnashequilibria} later
proved the same result (see
also the presentation of~\citeR{DaskalakisCACM,daskalakis:195}). This means that a
polynomial-time algorithm for this problem is unlikely. This seminal result  was
the culmination of a series of results that same year for $3$-player~\citep{DaskalakisThree,chenanddeng05_three},
and $4$-player~\citep{DaskalakisFour} normal-form
games. 

Within the game-theory community,~\citet{BubelisReduction} devised a polynomial-time reduction
of any $n$-player normal-form game to some $3$-player normal-form
game; in the theoretical computer science
community,~\citet{GoldbergReducibility} rediscovered a similar result,
reducing
$n$-player normal-form games to $4$-player normal-form games. This
last result rendered the computational complexity of computing a Nash equilibrium in
arbitrary $n$-player normal-form games as
PPAD-complete~\citep{GoldbergReducibility}. 

Several
results on the PPAD-completeness of computing Nash equilibria in arbitrary
GGs served as lemmas to the proofs for normal-form
games~\citep{DaskalakisFour,DaskalakisThree}. In fact, GGs played a very important role in
finally settling the computational complexity of bimatrix
games. 

There is other important, seminal work regarding the computation of Nash
equilibria prior to the hardness results discussed above,
including some regarding other equilibrium notions such as
pure-strategy Nash equilibria and correlated equilibria~\citep{aumann74,aumann87} in graphical
and normal-form games  (see,
e.g., the work of~\citet{gilboa89,mckelvey96,kakadeetal03,GottlobHardEasy,Conitzer2008621,Papadimitriou:2008:CCE:1379759.1379762,roughgarden09,DBLP:journals/sigecom/JiangL11},
and the references therein). But the focus of this
short paper is on representing and computing \emph{approximate Nash
  equilibria in GGs}.

\subsection{A Caveat: Exact vs. Approximate Nash Equilibria in Multiplayer
  Games} 

Technically, the result for more than $3$ players is for
the problem of computing an \emph{approximate} Nash equilibrium with an approximation
quality \emph{exponentially small} in the representation size of the
game. This is because, as is well-known since Nash's journal
version~\citep{nash51} of his original paper~\citep{nash50},
Nash presents a poker-inspired game of $3$ players, built in collaboration with
Lloyd Shapley,
a 2012 Nobel Laureate in Economics, in which all
payoff hypermatrix entries/values are \emph{rational} numbers, but whose
\emph{unique} Nash equilibrium contains mixed strategies with
\emph{irrational} numbers.

\citet{ChenApprox} also proved that there does not exist a
fully-polynomial-time approximation scheme (FPTAS) for computing an
approximate Nash equilibrium, unless PPAD=P (see also the
presentation of~\citeR{ChenSettlingJACM}).~\footnote{Formal definitions of the different
approximation schemes, which textbooks such as those
of~\citet{VaziraniAppAlgBook} and~\citet{WilliamsonAppAlgBook} cover,
are beyond the scope of this short paper. See Appendix~\ref{app:approx} for a
general description of some of the approximation schemes.} For this
result, they proved that the approximation
quality can not go below an \emph{inverse-polynomial} function of the
representation size of the game. They stated in their paper, without
proof, that it is easy to extend the result for $2$-player normal-form
games to $n$-player normal-form games and $n$-player graphical
games. \emph{Hence, computing a Nash equilibrium (or more formally, an
approximate Nash equilibrium with approximation quality
inversely-exponential in the representation size of the game), in an arbitrary
$n$-player GG is PPAD-complete.} 
Yet, it
is still open
whether there exists a polynomial-time approximation scheme (PTAS) for
$n$-player games in normal-form and arbitrary GGs. 

The result of~\citet{DaskalakisCACM} also implied that the problem
of computing an approximate Nash equilibrium in $n$-player normal-form games, for
$n \geq 2$, with approximation quality inversely
exponential on the game's representation size is PPAD-complete.

All of the above helps us motivate the study of approximate Nash equilibria in
arbitrary GGs from a technical perspective:
\begin{itemize}
\item[(a)] there are
multiplayer games with payoff matrices represented entirely using
rational numbers but with unique Nash equilibria
formed of mixed strategies with \emph{irrational} numbers; 
\item[(b)] computing
``exact'' Nash equilibria 
is likely to be computationally intractable in general; and 
\item[(c)]  while there is no FPTAS to compute an
approximate Nash equilibrium with approximation
quality below a value inversely polynomial of the representation size
of the game, the existence of a PTAS is open. 
\end{itemize}
Here, we provide an
array of results, from polynomial time, to FPTAS and quasi-PTAS, for a
variety of GGs under reasonable conditions on the
parameters and network characteristics of the GG.

\subsection{``Why revisit an old research note now?''}

There are several motivations for reviving the old
note~\citep{Ortiz_2002}.
\begin{itemize}
\item 
The note has gained significance and relevance with time. The sparse
discretization and main representation results are proving to be particularly
useful in establishing polynomial running times for a specific class
of dynamic-programming and propagation algorithms.

\item Many algorithmic results immediately follow
from the connection to the previous work on constraint
networks in AI. Thus, this paper highlights the usefulness and significance of
those results, coming from the AI community, regarding the use of constraint
networks to solve network-structured CSPs.

\item The derivation of the proof of
the result on sparse discretization is \emph{deliberately simple,} 
relative to considerably more mathematically complex proofs of the
same results (see, e.g., the derivations by~\citeR{Daskalakis_and_Papadimitriou_2008}): \emph{no fancy mathematical
sophistication needed when simple algebraic manipulations suffice.}
Similarly, this short paper highlights the important, but often ignored, role that the
work on constraint networks in AI has in
simplifying the derivation and analysis of
algorithms to compute approximate Nash equilibria in
GGs and some of their specializations, including
normal-form games.

\item This paper helps clarify the distinct pros and cons of
the two approaches most often used to compute approximate Nash Equilibria: sparse
support vs. sparse discretization.

\item This paper seeks to present the state-of-the-art on computing
  approximate Nash equilibria in GGs, and lists several important open problems, particular for
GG generalizations, with important practical
consequences, and for which the AI community is particularly
well-suited to eventually solve.

\end{itemize}
The following discussion expands on the first two bullet points, with
some emphasis on prior and current work in my research group or with collaborators, but mentioning briefly
potential implications for the work of other AI researchers.

\subsubsection{Games as CSPs}

The computation of Nash equilibria in games is inherently a CSP, an area of great
relevance to AI research. In the game-induced CSP, we have one
variable for each player's mixed strategy, one domain for each player
corresponding to the simplex over the player's actions, and a set of
constraints for each player's Nash equilibrium conditions.
For example, for binary-action games, recall that a mixed-strategy
is just a probability, real-valued number between
$0$ and $1$; thus, the number of mixed-strategies available to each
player is \emph{uncountable}. Hence, the corresponding CSP would have
real-valued variables, each with an uncountably infinite domain
size: an \emph{infinite} or \emph{continuous CSP}. The same holds for multi-action games.

As stated previously,~\citet{Kearns_et_al_2001} introduced the idea in AI of using a uniform
discretization of the mixed strategies of each player which leads to a
regular grid. This turns the problem of computing approximate Nash equilibria
into a \emph{discrete CSP}~\citep{Vickrey_and_Koller_2002}, i.e., a CSP
with finite, discrete domains. This paper considers the same type of discretization.

One motivation to revisit the old note~\citep{Ortiz_2002} is the implications it has on
extensions of {\bf NashProp}~\citep{ortizandkearns03} based on 
\emph{survey propagation (SP)}~\citep{mezard03,parisi-2003-a,parisi-2003-b,braunstein-2004-,braunsteinetal05,Maneva:2007:NLS:1255443.1255445}, a message-passing technique with roots in physics and introduced
to computer science to solve boolean satisfiability problems, such as random
$3$-SAT formulas~\citep{mezardetal02,parisi-2003-c,mezard-2004-}, and more generally, a variety of CSPs, such as graph $3$-coloring~\citep{parisi-2003-d}.~\citet{ortiz08} introduced an axiomatic way to
view the standard SP algorithm for boolean  satisfiability, and an
approach called \emph{constraint propagation relaxation
  (CPR)} that facilitates the derivation of SP-like
techniques tailored to specific CSPs. In the case of GGs,
the computation of each message-passing, between players, takes
\emph{exponential} time in the size of the discretization of the space
of mixed strategies. (A thorough description of the SP-like version of
{\bf NashProp} is beyond the scope of this paper.) But the improved bound on the discretization size
for each player
is \emph{logarithmic} in the representation size
of the GG, \emph{assuming a constant number of actions}, leading
to a running time that is polynomial in $1/\epsilon$, where $\epsilon$
is the approximation-quality parameter, and either quasi-polynomial in the
representation size of the game in the case of arbitrary maximum neighborhood
size $k$, or simply polynomial in the case $k$ is bounded by a
constant, independent of the number of players $n$. The representation
size of each message is simply \emph{quadratic} in $k$, $m$ and
$1/\epsilon$, even if $k$ and $m$ are free parameters.

Another motivation for considering GG generalizations is
that they include the class of \emph{graphical poly-matrix games}, which in the
case of $2$-action games relates to \emph{linear (or generalized linear)
influence
games}~\citep{IrfanAAAI,IrfanTechReport,IrfanPhD}. In part because of their
very compact representations (i.e., of size $O(n \, m^2)$), \emph{poly-matrix
games}~\citep{Janovskaja_1968_MR_by_Isbell} are an important class of
games within game theory. We revisit their graphical version when we
consider generalizations of GGs in Section~\ref{sec:gen}.

In summary, particular
interest in \emph{the application of CPR to derive better algorithms
for computing and counting Nash equilibria in arbitrary, loopy
GGs,} as well as more specific GG classes such as \emph{linear or
generalized-linear influence
games}~\citep{IrfanAAAI,IrfanTechReport,IrfanPhD}, motivates this
paper in large part.

Similarly, the new bounds may also provide improvements to previous
discretization-based schemes for computing $\epsilon$-Nash equilibria
in other similar models (e.g., those of~\citeR{Singh_et_al_2004,Soni_et_al_2007}).

\section{Preliminaries}
\label{sec:prelim}

This section introduces the basic technical notation and concepts
necessary to understand the upcoming technical sections of the
research note.

\subsection{Basic Notation}

Denote by $x \equiv (x_1,x_2,\ldots,x_n)$ an $n$-dimensional vector
and by $x_{-i} \equiv (x_1,\ldots,x_{i-1},x_{i+1},\ldots,x_n)$ the
same vector without component $i$. Similarly, for every set $S \subset
[n] \equiv \{1,\ldots,n\}$, denote by $x_S \equiv (x_i : i \in S)$ the
(sub-)vector formed from $x$ using only components in $S$, such that,
if $S^c \equiv [n] - S$ denotes the complement of $S$, $x \equiv
(x_S,x_{S^c}) \equiv (x_i,x_{-i})$ for every $i$. If $A_1,\ldots,A_n$
are sets, denote by $A \equiv \times_{i \in [n]} A_i$, $A_{-i} \equiv
\times_{j \in [n] - \{i\}} A_j$ and $A_S \equiv \times_{j \in S}
A_j$. 

If $G=(V,E)$ is an undirected graph, then for each $i \in V$ denote by $\Neigh_i \equiv \{ j
\mid (j,i) \in E\}$ the \emph{neighbors of node/vertex $i$} in
$G$, and $N_i \equiv \Neigh_i \bigcup \{ i \}$ the \emph{neighborhood of
  node/vertex $i$} in $G$. Note that we have $i \notin \Neigh_i$ but $i
\in N_i$ for all $i \in V$.

\subsection{Graphical Games in Local Normal-Form Payoff Representations}

This section formally defines \emph{graphical games (GGs)}, which are graphical
models for compact representations of classical game representations
in game theory~\citep{Kearns_et_al_2001}. GGs extend and generalize \emph{normal-form
games.} In particular, a normal-form game is a GG with a complete/fully-connected graph. 

\begin{definition}
A {\em graphical game (GG)\/} consists of an \emph{undirected graph}
$G=(V,E)$,~\footnote{It is easy to extend the
  same result to GGs with \emph{directed} graphs.} where
each node $i \in V$ in the graph corresponds to a player $i$ in the
game, and for each player $i$, we have a set of \emph{actions} or pure
strategies $A_i$ and a local payoff hypermatrix/function $M'_i : A_{N_i}
\to \bbbr$, where $N_i$ is the \emph{neighborhood} of player $i$ in the
graph defined with respect to the edges $E$ of the graph.
The (global) payoff hypermatrix/function $M_i$ of player $i$ is such
that, for each joint-action $x \in A \equiv A_V$, we have
$M_i(x) \equiv M'_i(x_{N_i})$. That is, the payoff that each individual
player $i$ receives when all players, including $i$, take joint-action/pure-strategy
$x$ is a \emph{function of the joint-actions $x_{N_i}$ of player $i$'s
  neighborhood $N_i$ only, thus conditionally independent of
  $x_{V-N_i}$ given $x_{N_i}$}. It
is convention to let $V = \{1,\ldots,n\} \equiv [n]$, so that $n
\equiv |V|$. The representation size of each local
payoff hypermatrix $M'_i$ is $\Theta(|A_{N_i}|) = O(m^k)$, where $m \equiv \max_{i
  \in V} |A_i|$ and $k \equiv \max_{i \in V} |N_i|$. The
representation size of the GG is $\Theta(\sum_{i \in V} |A_{N_i}|)
= O(n m^k)$.
If 
for all $i$
we have $N_i = V$, then the GG is a standard {\em
  normal-form game}, also called strategic- or matrix-form game,
which has a representation size $\Theta(n |A|) = O(n m^n)$.
\end{definition}
A GG achieves considerable representation savings whenever
$k \ll n$.

\subsection{Solution Concepts}

A {\em joint mixed strategy\/} $p \equiv (p_1,\ldots,p_n)$ in a game is formed from each individual {\em mixed strategy\/} $p_i \equiv (p_i(x_i) : x_i \in A_i)$ for player $i$, which is a probability distribution over the players actions $A_i$ (i.e., $p_i(x_i) \geq 0$ for all $x_i \in A_i$ and $\sum_{x_i \in A_i} p_i(x_i) = 1$). Denote by $\mathcal{P}_i \equiv \{ \, p_i \mid p_i(x_i) \geq 0, \mbox{ for all } x_i \in A_i \mbox{ and } \sum_{x_i \in A_i} p_i(x_i) = 1\}$ the set of all possible mixed strategies of player $i$ (i.e., all possible probability distributions over $A_i$).  Similar to the vector notation introduced above, for all $i$ and any clique/set $S \subset V$, denote by $p_{-i}$ and $p_S$ the mixed strategies corresponding to all the players {\em except\/} $i$ and all the players in clique $S$, respectively, so that $p \equiv (p_i,p_{-i}) \equiv (p_S,p_{V-S})$.
A joint mixed strategy $p$ induces a joint (product) probability
distribution over the joint action space $A$,
such that, for all $x \in A$, $p(x) \equiv \prod_{i \in V} p_i(x_i)$ is the probability, with respect to joint mixed strategy $p$, that joint action $x$ is played.

The {\em expected payoff\/} of player $i$ with respect to joint mixed strategy $p$ is denoted by $M_i(p) \equiv \sum_{x \in A} p(x) M_i(x)$.

\begin{definition}
For any $\epsilon \geq 0$, a joint mixed-strategy $p^*$is called an
{\em $\epsilon$-Nash equilibrium\/} if for every
player $i$, and for all $x_i \in A_i$, $M_i(p^*_i,p^*_{-i}) \geq
M_i(x_i,p^*_i) - \epsilon$. That is, no player can increase its
expected payoff more than $\epsilon$ by {\em unilaterally\/} deviating
from its mixed strategy part $p^*_i$ in the equilibrium, assuming the
others play according to their respective parts $p^*_{-i}$. A {\em
  Nash equilibrium\/}, or more formally, a mixed-strategy Nash
Equilibrium, is then a 0-Nash equilibrium.
\end{definition}
Note that, for all $p_{-i} \in \mathcal{P}_{-i}$, $\max_{p_i \in \mathcal{P}_i} M_i(p_i,p_{-i}) = \max_{x_i \in A_i} M_i(x_i,p_{-i}) \geq  M_i(x'_i,p_{-i})$, for all $x'_i \in A_i$.
Also, note that the equilibrium conditions are invariant to affine
transformations. In the case of GGs with local payoff
matrices represented in tabular/matrix/normal-form, it is convention
to assume, without loss of generality, that the payoff
values are such that, for each player $i \in V$, we have
$\min_x M_i(x) = \min_{x_{N_i}} M'_i(x_{N_i})
= 0$ and $\max_x M_i(x) = \max_{x_{N_i}} M'_i(x_{N_i})
= 1$. Note that in the case of GGs using such ``tabular''
representations, we do not lose generality by assuming the maximum and minimum local
payoff values of each player are $0$ and $1$, respectively, because we
can compute them both \emph{efficiently.} This will \emph{not} be the case for
GG generalizations, in the worst
case. Section~\ref{sec:gen} revisits this point.

\section{Discretization Schemes}

The discretization scheme considered here is similar to that of~\cite{Kearns_et_al_2001}, except that we allow for the possibility of different discretization sizes for the mixed strategies of players. 
\begin{definition}
In an {\em (individually-uniform) discretization scheme\/}, the uncountable set $I = [0,1]$ of possible value assignments to the probability $p_i(x_i)$ of each action $x_i$ of each player $i$ is approximated by a finite grid defined by the set $\widetilde{I}_i = \{0, \tau_i, 2 \tau_i, \ldots, (s_i-1) \tau_i, 1\}$ of values separated by the same distance $\tau_i = 1/s_i$ for some integer $s_i$. Thus the {\em discretization size\/} is $|\widetilde{I}_i| = s_i + 1$. Then, we would only consider mixed strategies $q_i$ such that $q_i(x_i) \in \widetilde{I}_i$ for all $x_i$, and $\sum_{x_i \in A_i} q_i(x_i) = 1$. The {\em induced discretized space\/} of joint mixed strategies is $\widetilde{I} \equiv \times_{i \in V} \widetilde{I}_i^{\, |A_i|}$, subject to the individual normalization constraints.
\end{definition}

\section{Sparse Discretization}

The obvious question is, how small can we make $s_i$ and still
guarantee that there exists an $\epsilon$-Nash equilibrium in the induced
discretized space of joint mixed strategies? The following 
corollary provides a
stronger answer for GGs: it provides values
for the $s_i$'s that guarantee that for \emph{every} Nash equilibrium, its closest
point in the induced grid is an $\epsilon$-Nash equilibrium. An interesting aspect
of the result is that $s_i$ depends only on information local to
player $i$'s neighborhood: the number of actions $|A_i|$
available to player $i$ and the largest number of neighbors $|\Neigh_j|$
\emph{of the neighbors} $j \in \Neigh_i$ of
player $i$.

Note that the corresponding discretization bound provided in~\cite{Kearns_et_al_2001} in the context of GGs is \emph{exponential} in the largest neighborhood size $k$. In contrast, the bound here is {\em linear\/} in $k$, a substantial reduction.

The corollary is a GG instantiation of the
Nash-equilibria Representation Theorem based on sparse discretization, Theorem~\ref{the:spdisc}, which
holds for a broader class of GG generalizations. 
 The
statement and discussion of
the theorem is in Section~\ref{sec:gen}. The statement of the corollary uses notation introduced above.
\begin{corollary}
{\bf(Sparse Discretization for Graphical Games)} 
\label{cor:spdisc}
 For any $m$-action graphical game and any $\epsilon > 0$, a
 (individually-uniform) discretization with \[ s_i = \left\lceil \frac{2 \, |A_i|
     \, \max_{j \in \Neigh_i} |\Neigh_j|}{ \epsilon }\right\rceil
 = O\left( \frac{m \, k}{\epsilon} \right) \] for each player $i$ is sufficient to guarantee that for {\em every\/} true (i.e., not approximate) Nash equilibrium of the game, its closest (in $\ell_\infty$ distance) joint mixed strategy in the induced discretized space 
is also an $\epsilon$-Nash equilibrium of the game.
\end{corollary}

\section{Algorithmic Implications}
\label{sec:alg}

One objective of this research note is to clarify the distinction
between sparse discretization and sparse support in the context of
approximate Nash equilibria. The next subsection deals with that objective in the context of normal-form games. Then,
the second subsection presents the algorithmic implications of the
sparse-discretization approach, borrowing heavily from standard
results in AI, and highlighting their simplifying usefulness and powers.

Before starting the presentation, it is important to note a crucial
distinction between the parameters of interest in GGs
versus those of normal-form games, a specialization of GGs. In the study of GGs, the main objects of
interest, in terms of input-game representation size,
is the number of players $n$ and the graph structure (e.g., the
maximum number of neighbors $k$, the graph tree-width, or the
hypertree-width of the game-induced constraint network), and their
effect in both representation and computations of (approximate) Nash
equilibria; the maximum number of actions $m$ of each player plays a
lesser role, and is
often assumed constant. Yet, the statements of the results here still
sometimes include the explicit
dependence on $m$, so that such dependence is clear should $m$ also be
an important component/aspect of interest in evaluating the quality of
the approximations and the resulting algorithms.

\subsection{Sparse Discretization vs. Sparse/Small Support}
\label{sec:distinct}

Before starting this subsection, it is important to clarify that the work of~\citet{Althofer1994339}
  does not imply, and cannot be used to derive, the sparse representation
  result presented as the main technical contribution here. Instead, the approximation bounds presented in~\citet{Lipton:2003:PLG:779928.779933}
  follow immediately from those of~\citet{Althofer1994339}; in fact,
  both use
  the same approach based on applying Ho\"effding bounds~\citep{Hoeffding63} and the
  probabilistic method~\citep{Alon_and_Spencer_2004}. More
  recently,~\citet{HemonANE} derived a slightly improved small-support
  upper-bound result in multi-player games by using a more general large
  deviation bound than Ho\"effding's, due to~\citet{McDiarmid}.

It is equally important to understand the distinctions between the two
types of approximations. Each type has their own pros and cons. 

As a warmup to the results and discussion regarding GGs and their
generalization, let us first consider GG simplifications,
starting with the simple case of $2$-player games in normal form, and
then moving to $n$-player games in normal-form.

\subsubsection{Bimatrix games}

In the
small-support approach to approximation, we perform a brute-force support-enumeration
search over all possible ${m \choose r}$ subsets of strategies, where
$r$ is the minimum
support necessary to guarantee the desired absolute approximation
quality $\epsilon > 0$; thus, it is
important to keep in mind that $r$ depends on $\epsilon$. Hence, the
running time of approximation algorithms of that type is ${m \choose r}
= O(m^r)$, thus polynomial in $m$ but \emph{exponential} in a
function of the approximation parameter. Given
$\epsilon > 0$, the support size
that results from the approach
of~\citet{Althofer1994339} and~\citet{Lipton:2003:PLG:779928.779933}
is $r = O\left(\frac{\ln{(m \, n)}}{\epsilon^2} \right)$. It is a
well-known fact that given the supports for each player in a Nash equilibrium, one
can set up a simple linear program (LP) to find the Nash equilibrium. Each call to
the LP runs in polynomial time in $r$. This leads to a naive
  brute-force algorithm, also called ``support enumeration,'' for
  computing $\epsilon$-Nash equilibrium in bimatrix games:
  for each possible support set of size $r$, run an LP to either find
  an $\epsilon$-Nash equilibrium or decide there is none with that particular
  set. There are ${m \choose r}
= O(m^r)$ possible support for each player. Hence, for
  $2$-player games, the algorithm runs in time
  $m^{O(r)} = m^{O\left({\frac{\ln{m}}{\epsilon^2}} \right)}$, thus
\emph{exponential} in $1/\epsilon^2$, and also exponential in
$(\ln{m})^2$, which means \emph{quasi-polynomial} in $m$.  

But, there is actually an alternative to the LP. Because of the probabilistic way in which the result
for sparse supports arises (i.e., using a combination of Hoeffding's or other
large deviation bound and invoking the probabilistic method), we only
need to search over mixed strategies in the support whose individual
probability value has the form $z_k/r$, for each action $k =
1,\ldots,m$, and $\sum_k z_k/r = 1$, $0 \leq z_k \leq
r$.~\citet{Lipton:2003:PLG:779928.779933} call such ``discretization scheme'' $r$-uniform.  Another way
of viewing this is as a uniform-discretization scheme of size
$r+1$, which does not seem to guarantee to find all Nash equilibria, but it does guarantee
to find at least one. Thus, in that case, the worst-case running time
to find at least one $\epsilon$-Nash equilibrium is $(m \, r)^{O\left( r
\right)}$. Recalling that $r = O\left(\frac{\ln{m}}{\epsilon^2}
\right)$, we have that the worst-case running time is still
\emph{exponential} in $1/\epsilon^2$ and exponential in $(\ln{m})^2$, thus
\emph{quasi-polynomial} in $m$. 

Regardless, because the representation size of the game is
$N=O(m^2)$, both algorithms run in time
$N^{O\left(\ln{N}\right)}$, thus \emph{both are quasi-PTASs for computing approximate
Nash equilibria in bimatrix games}~\citep{Lipton:2003:PLG:779928.779933}.

 In contrast, suppose
    that $s$ is the size of the uniform-grid sparse discretization
    presented here. A simple analysis reveals $s
    = O\left(\frac{m}{\epsilon} \right)$. Hence, using the
      sparse discretization, a naive brute-force search runs in time
      $O\left(\left( \frac{m}{\epsilon} \right)^{m}\right)$, thus
\emph{polynomial} in $1/\epsilon$, but \emph{exponential} in $m
\ln{m}$. 

\subsubsection{Multiplayer games}
\label{sec:multi}

For multiplayer games, one cannot
use an LP to compute an $\epsilon$-Nash equilibrium even if one were to know the
support of the $\epsilon$-Nash equilibrium strategies of all
players. This is because now the $\epsilon$-Nash equilibrium conditions are highly
non-linear, involving the product of $n$ variables, where $n$ is the
number of players, each variable corresponding to the probability of a
particular action in the mixed strategy of a player. (Recall that a
Nash equilibrium
is a product distribution.) Thus, it might
first appear unclear how to efficiently perform this computation
efficiently in multiplayer games. It would seem, on the surface, that for sparse support, an algorithm
based on support enumeration would run in worst-case time $O\left(m^{r
    \, n} \mathrm{rtime}(n,m,r) \right)$, where $\mathrm{rtime}(n,m,r)$ is the worst-case running time of the
algorithm that attempts to compute an $\epsilon$-Nash equilibrium with the given
players' supports, and is a direct function of $r,m$, and $n$, and of
course, indirectly, of
$\epsilon$ too. 

But, as mentioned in the case of bimatrix games, because of the probabilistic way in which the result
for sparse supports arises, we only
need to search over the set of $r$-uniform mixed strategies in the
support. Thus, using this approach of support enumeration the
worst-case running time becomes $\mathrm{rtime}(n,m,r) = O\left( {m \choose r} \, r^r
   \right)$. Thus, the total running time is $O\left((mr)^{r
    \, n} \right)$. ~\citet{Lipton:2003:PLG:779928.779933} proved that
using $r = O\left( \frac{n^2 \ln{(n^2 m)}}{\epsilon^2}
\right)$ is sufficient.~\footnote{The reader should be aware that the notation of~\citet{Lipton:2003:PLG:779928.779933} is the exact opposite of
  the one here: here $n$ denotes the number of players, while there it
  denotes the maximum number of actions of any player; also, here $m$
  denotes the maximum number of actions of any player, while there it
  denotes the number of players.} Substituting that value of $r$, the
expression of the worst-case
running time becomes $m^{O\left( \left( n^3 \ln{(n^2 \, m)}
      \right) \left( \frac{1}{\epsilon^2} \ln{\frac{1}{\epsilon^2}}
      \right) \right)}$. Now, to obtain a quasi-PTAS, 
we would have to impose a condition on $n$, such as $n =
O(\ln{m})$. This is unlike the bimatrix-game case
in which no restriction was necessary.

In contrast, for sparse discretization in multiplayer games, the value
of $s
    = O\left(\frac{m \; n}{\epsilon} \right)$. Hence, a brute-force search algorithm applied to a
normal-form game, for example, would run in
worst-case time $\left( \frac{m \, n}{\epsilon} \right)^{O\left( m \,
            n\right)}$, thus
\emph{polynomial} in $1/\epsilon$, but \emph{exponential} in $m \, n
\ln{m \, n}$. (The results for GGs are in
        a later section.)  To obtain a quasi-PTAS, we would have to
        impose a condition on $m = \mathit{poly}(n)$, not $n$.

\subsubsection{Tradeoff}

The tradeoff is now clear. If the interest is
        $m$, and $\epsilon$ is fixed, then the approach
        of~\citet{Althofer1994339}
        and~\citet{Lipton:2003:PLG:779928.779933} is better.
        On the other hand, if the interest is $\epsilon$, and $m$ is fixed,
        which is often the case for GGs as discussed above, then the sparse
        discretization wins. It is important to note that the sparse
        discretization guarantees the computation of \emph{all}
        $\epsilon$-Nash equilibria, while the sparse-support approach can only
        guarantee \emph{one} $\epsilon$-Nash equilibrium. It is also important to
        note that in the case of multiplayer games, the sparse-support
        approach yields a quasi-PTAS only after bounding $n$, which
        is the quantity of most interest in GGs. The
        sparse-discretization approach bounds $m$, which is the number
        of actions of each player to obtain a quasi-PTAS, but $m$ is
        not as significant in a GG setting as $n$ is.

\subsection{Reductions to Consistency in CSPs}

This section involves concepts from AI and graph
theory; in the interest of space, the reader is referred to
appropriate standard references
(see, e.g., the textbooks of~\citeR{Russell_and_Norvig_2003}, for AI
and~\citeR{Dechter_2003} for graph theory as used in AI; or~\citeR{bollobas} for graph theory).

The representation result of the last section has several immediate
computational consequences for the problem of computing approximate
Nash equilibria in GGs with local payoff matrices represented in tabular form, and in turn, also for multi-player games represented in standard normal (tabular) form.

To simplify the presentation, let us assume that payoff values are in $[0,1]$, all the players in the GG have the same number of actions $m$ and the largest neighborhood size in the graph of the game is $k $,
so that the representation size of the GG is $\Theta(n m^k)$. For this case, the uniform discretization presented in Theorem~\ref{the:spdisc} has size $s = s_i = O(m k / \epsilon)$ for all $i$.

\subsubsection{Sparse-discretization GG-induced CSP}

Once we introduce a discretization over the space of mixed strategies,
then it is natural to formulate the problem of computing
$\epsilon$-Nash equilibria
on the induced discretized space as a CSP, or more specifically in the
case of GGs, as a special type of
constraint network~\citep{Dechter_2003}. (In the interest of keeping this paper short, please
see~\citeR{Russell_and_Norvig_2003}, or other introductory textbook on
AI, for general information on CSPs. The presentation here contains
only the CSP concepts
necessary exclusively within the context of the paper's topic.) Several researchers have taken this or related approaches either explicitly or implicitly~\cite{Kearns_et_al_2001,Vickrey_and_Koller_2002,ortizandkearns03,Soni_et_al_2007}. The CSP for the game has one variable, domain and constraint for each player. Each variable corresponds to mixed strategy $p_i$ for each player $i$. Each variable's domain corresponds to the discretized set $\widetilde{I}^m_i$ of mixed strategies for each player $i$, properly corrected to account for normalization. The \emph{approximate best-response equilibrium conditions} are the following:
for each player $i$, each constraint function (table) $c_i : \widetilde{I}^{m k}_i \to \{0,1\}$ is defined such that, for all $p_{N(i)} \in \widetilde{I}^{m k}_i$, $c_i(p_{N(i)}) = 1$ if and only if $M_i(p_i,p_{-i}) \geq  \max_{x_i' \in A_i} M_i(x'_i,p_{-i}) - \epsilon$.
Each constraint has arity at most $k$ and encodes the approximate best-response equilibrium conditions, each represented in tabular form using $s^{m k} = O((m k/ \epsilon)^{m k})$ bits.
The transformation takes time 
$O(\mathit{poly}(n, (m k/\epsilon)^{m k}))$ 
and the size of the resulting CSP is $T =\Theta(n (m k/\epsilon)^{m k})$.

As previously discussed, for GGs, it is natural to consider the number of players $n$ as being
the ``free'' parameter of the representation. Hence, if  $m k \log(m
k) = O(\log(n))$ and $\epsilon = n^{\Omega(-1/(m k))} = (m
k)^{-\Omega(1)}=(\log n)^{-\Omega(1)}$, then both the time to perform
the CSP transformation and its resulting representation size are
polynomial in the representation size of the game (i.e., polynomial in
the number of players). Similarly, if $m$ and $k$ are bounded (by a
constant, independent of $n$), then the representation size of the game is linear in $n$ and the running time of the transformation is polynomial in $n$ and $1/\epsilon$. Note that this is a natural restriction on the game parameters because otherwise the representation size would be exponential in the number of players, thus defeating the main purpose of 
the GG representation 
in the first place: \emph{succinctness.}

At this point, we can apply any of a large number of existing
off-the-shelf techniques for solving the induced game-CSP, or apply
techniques such as {\bf NashProp}~\cite{ortizandkearns03} that
take advantage of the particular properties of the game best-response
constraints. 

Instead, in the next section, we will see how standard results for solving constraint
networks~\citep{Dechter_2003} lead immediately to simple derivations of algorithmic results
for GGs, some of which had been independently re-discovered by employing more
sophisticated mathematical tools~\citep{Daskalakis_and_Papadimitriou_2008}. Facilitated in large part by the knowledge acquired
for solving constraint networks and other graphical models in the AI community over the last 50
years, the derivations here are quite straightforward and do not require
complex mathematics. By building on existing AI knowledge, fancy
mathematical derivations for the same result become unnecessary despite their elegance.

\subsubsection{An approach based on the GG-induced CSP hypergraph}

An approach to solving CSPs in AI, now over 10 years olds, works on
the \emph{hypergraph induced by the game
  CSP}~\cite{Gottlob_et_al_2001}. If $T$ is the representation size of
the game-CSP, $w$ is the hypertree width of the hypergraph, and the
corresponding hypertree decomposition for the CSP has been computed,
then solving the CSP takes time
$O(T^{w+1}\log(T))$~\citep{Gottlob_et_al_2000,Gottlob_et_al_2002,Gottlob_et_al_2001}
(see also page 158 of ``Bibliographical and Historical Notes'' Section
in Chapter 5 of~\citeR{Russell_and_Norvig_2003}). In the case of
GGs, we can
compute the hypertree decomposition that the algorithm would use in time $O(n^{2 w + 2})$.

The following theorem summarizes the discussion.
As a preamble to a discussion on primal graphs and treewidths later in
the section, note that it is known that a CSP might have hypergraphs
with bounded hypertree width, but whose \emph{primal} graph has {\em
  unbounded\/} treewidth. However, the treewidth \emph{always bounds}
the hypertree width. Thus, the restriction that the hypertree
width be bounded by a constant may not be as limiting to the
application of the results as it first appears.
\begin{theorem}
\label{the:hw}
There exists an algorithm that, given as input a number $\epsilon > 0$ and a GG with $n$ players, maximum neighborhood size $k$ and maximum number of actions $m$, and whose corresponding CSP has a hypergraph with hypertree width $w$, computes an $\epsilon$-Nash equilibrium of the GG in time $[ n \left( m k / \epsilon \right)^{m k} ]^{O(w)}$.
\end{theorem}
The following corollary characterizes the computational complexity of the
approximation schemes resulting from instances of the last theorem.
\begin{corollary}
There exists an algorithm that, given as input a GG with corresponding
hypergraph of hypertree width $w$ bounded by a constant independent
of the number
of players $n$, with a
logarithmic function of $n$ restricting the maximum number of 
actions $m$ and the maximum neighborhood size $k$ as $m k \log(m k) =
O(\log(n))$, and given $\epsilon = n^{\Omega(-1/(m k))} =
(mk)^{-\Omega(1)} = (\log{n})^{-\Omega(1)}$, outputs an $\epsilon$-Nash equilibria of the game in time polynomial in the representation size of the input. 
If, in particular, both $m$ and $k$ are bounded by constants independent
of $n$, then the
algorithm runs in polynomial time in $n$ and $1/\epsilon$, for any
$\epsilon > 0$; hence, the algorithm is a {\em fully polynomial time
  approximation scheme (FPTAS)}. If, instead, the expression
constraining $m$ and $k$ as a logarithmic function of $n$ holds, and $w = \mbox{polylog}(n)$, then the algorithm is a {\em quasi-polynomial time approximation scheme (quasi-PTAS)}.
\end{corollary}

\subsubsection{A side note on
normal-form games}

For normal-form games, $k = n$ and $w = 1$. This leads to the following corollary.
\begin{corollary}
There exists a quasi-PTAS for computing an $\epsilon$-Nash equilibrium
of $n$-player $m$-action normal-form games with $m =
O(\mathit{poly}(n))$ that runs in time $N^{O(\mathit{polylog}(N)
  \log(1/\epsilon))} = \left( \frac{1}{\epsilon} \right)^{O(\mathit{polylog}(N))}$, where $N = n^{\Theta(n)}$ is the representation
size of the game. If, in particular, $m$ is bounded by a constant independent
of $n$, then the running time is $N^{O(\log \frac{\log(N)}{\epsilon})}$, where $N = 2^{\Theta(n)}$ is the corresponding representation size of the game.
\end{corollary}
As briefly mentioned in Section~\ref{sec:multi}, we can also
obtain the same result by using the sparse-support approach of
\citet{Lipton:2003:PLG:779928.779933}, even if 
$m = 2^{O(n)}$, but the dependence
is {\em exponential\/} in $1/\epsilon^2$ (i.e.,
$N^{O(\mbox{polylog}(N) (1/\epsilon)^2)}$; or $N^{O(\frac{\log(N)}{
    \epsilon^2})}$, if $m$ is bounded by a constant). The result of the last
corollary for the case of $m$ fixed or bounded by a constant is
stated by
\citet{Daskalakis_and_Papadimitriou_2008}. The same result for $m$
fixed or bounded by a constant also follows from Theorem 4
 of~\citet{Kearns_2007}. While no formal proof appears for Theorem 4
 of~\citet{Kearns_2007}, the theorem
  follows immediately from the proof in the original note of
  \citet{Ortiz_2002}. 

But, it is important to emphasize, as touched upon in Section~\ref{sec:multi}, that in the case of normal-form games, we could have obtained the result directly by
using an exhaustive search over the induced grid over mixed strategies, which is essentially what the algorithm
referred to in the corollary reduces to in this case. Hence, we could
output not just one $\epsilon$-Nash equilibrium, but \emph{all}
$\epsilon$-Nash equilibria in
the induced grid in the worst-case running time given in the
corollary. The algorithms based on the sparse-support approach can
only guarantee to output \emph{one} $\epsilon$-Nash equilibrium among all mixed strategies of a
given maximum support. Algorithms based on sparse support that would output
(a compact representation of)
\emph{all} $\epsilon$-Nash equilibria do not seem to exist, even for the given maximum support size; let
alone (a compact representation of) all  $\epsilon$-Nash equilibria in the game,
as the exhaustive search that uses the sparse discretization result
presented here does for normal-form games. Section~\ref{sec:distinct} here
discusses the distinctions between the sparse-support and the
sparse-discretization approaches. 

\subsubsection{An approach based on the GG-induced CSP primal graph}

Another approach is to build a clique (or join) tree from the {\em
  primal graph\/} of the game-CSP, which in the case of the GGs is the
graph created by forming cliques of every neighborhood. Then, one
applies a dynamic programming (or message-passing) algorithm on the
clique tree. Once the clique tree is built, the running time is linear
in the number of nodes of the join tree and exponential in the size of
the largest clique associated to a node in the clique tree. If the
primal graph has treewidth $w'$, the largest clique associated to the
optimal clique tree is $w' + 1$. It has been common knowledge in the
graphical-models community for quite a while now, almost two decades,
at least (see, e.g., the textbooks of~\citeR{Dechter_2003,Russell_and_Norvig_2003}, and the references therein, for a recent accounting of previous work in this area) that if the primal graph of a CSP has treewidth $w'$ that is logarithmic in the number of nodes $n$, then one can solve the CSP in polynomial time if the CSP is represented in tabular form. 
The following theorem and corollary follow from careful application of previously known results for solving constraint networks and other related graphical models.
Note that the treewidth $w''$ of the original graph of the game is
always no smaller than the hypertree width $w$ of its
hypergraph~\cite{GottlobHardEasy}. In addition, the GG's primal graph
treewidth $w' \leq (w" + 1)
k$~\cite{Daskalakis_and_Papadimitriou_2006}. So the interesting bounds
in the hypertree case is as given in the corollary. Also, this means
that the results can be easily extended to GGs with (original) graphs
that have $O(\log(n))$ treewidth as long as $k$ is bounded. Finally,
if a graph with $n$ nodes has treewidth $w'$, then the graph has at
most $(w' + 1) n$ edges (see, e.g., the work of~\citeR{Becker_and_Geiger_2001}). Because the number of edges of a GG primal graph is $O(k^2 n)$, $w' = O(\log(n))$ implies $k = O(\sqrt{\log(n)})$.
\begin{theorem}
There exists an algorithm that, given as input a number $\epsilon > 0$
and an $n$-player $m$-action GG with maximum neighborhood size $k$ and
primal-graph treewidth $w'$, computes an $\epsilon$-Nash equilibrium of the game in time $2^{O(w')} n \log(n) + n [ \left( m k / \epsilon \right)^{m k} ]^{O(w')}$. 
\end{theorem}
\begin{corollary}
There exists a PTAS for computing an approximate Nash equilibria in $n$-player GGs with bounded maximum number of actions, bounded neighborhood size and primal-graph treewidth $w' = O(\log(n))$.
\end{corollary}

The new discretization bounds also provide significant improvements on the representation results and running times for {\bf NashProp} and its variants~\cite{Kearns_et_al_2001,ortizandkearns03}.

\section{Graphical Multi-hypermatrix Games: Generalizing Graphical Games}
\label{sec:gen}

This section introduces {\em graphical multi-hypermatrix games (GMhGs)\/}, a class of games that extends and generalizes GGs while capturing many classical game-theoretic model representations, as discussed below.
This class of games is not some theoretical concoction: 
they are not only convenient in their generality, covering a large number of existing models, but also practical.
Indeed, \citet{Yu_and_Berthod_1995} used the same type of games to
establish an equivalence between {\em local\/} maximum-{\em
  a-posteriori\/} (MAP) inference in Markov random fields and the Nash
equilibria of the induced game.
~\footnote{Unbeknown at the time, the MRF-induced game is a
  \emph{potential game}~\citep{Monderer_and_Shapley_1996}; therefore
  {\em sequential\/} (or {\em synchronous\/}) best-response dynamics
  converges. Because, in addition, the goal in MAP inference is to
  obtain a {\em global\/} optimum configuration, in an attempt to
  avoid local minima, \citet{Yu_and_Berthod_1995} proposed a
  Metropolis-Hastings-style algorithm, which is also similar to
  simulated annealing algorithms used for solving satisfiability
  problems, and other local methods such as WalkSAT. We can view their
  proposed method as a kind of \emph{learning-in-games}
  scheme~\citep{Fudenberg_and_Levine_1999} based on best-response with
  random exploration, or ``trembling hand'' best-response, in which,
  at every round, each player individually play some best-response,
  with some probability; otherwise the player replays the
  player's previous/last response.}

This section also states and discuss the core theorem on sparse
discretization in this broader class generalizing GGs.
The simplicity, broadness and strength of impact makes this extension
theorem the major technical contribution of this research note.
\begin{definition}
A {\em graphical multi-hypermatrix game (GMhG)\/} is defined by a set
$V$ of $n$ players, and for each player $i \in V$, a set of {\em
  actions\/}, or pure strategies, $A_i$; a set $\mathcal{C}_i \subset
2^V$ of {\em cliques\/}, or hyperedges, such that if $C \in \mathcal{C}_i$ then $i \in C$; and a set $\{M'_{i,C} : A_C \to \bbbr  \mid C \in \mathcal{C}_i \}$ of {\em local-clique payoff matrices\/}. For each player $i \in V$, the sets $N(i) \equiv \cup_{C \in \mathcal{C}_i} C$ and $\Neigh_i \equiv \{ j \in V \mid i \in N(j), j \neq i\}$ are the clique of players affecting $i$'s payoff including $i$ (i.e., $i$'s neighborhood) and those affected by $i$ not including $i$, respectively.  
The {\em local\/} and {\em global payoff matrices\/} $M'_i : A_{N(i)} \to \bbbr$ and $M_i : A \to \bbbr$ of $i$ are (implicitly) defined as $M'_i(x_{N(i)}) \equiv \sum_{C \in \mathcal{C}_i} M'_{i,C}(x_{C})$ and $M_i(x) \equiv M'_i(x_{N(i)})$, respectively. 
\end{definition}
\paragraph{Connections to other game classes.}
If for each player, each clique set is a singleton, we obtain a {\em
  graphical game}, and the single clique in the set defines the
neighborhood of the player (i.e., in that case, $\mathcal{C}_i = \{
N(i) \}$ for all $i$). Furthermore, if, in addition, each clique is
the complete set of players, then the game is a standard {\em
  normal-form game}, also called strategic- or matrix-form game (i.e.,
in that case, $N(i) = V$ for all $i$). A GMhG becomes a classical,
standard {\em polymatrix game\/}~\citep{Janovskaja_1968_MR_by_Isbell}
if for each player $i$, $\mathcal{C}_i = \{ \{i,j\} \mid j \in V, j
\neq i \}$, which is the set of cliques of pairs of nodes involving
the player and every other player.~\footnote{Note that this is the
  standard definition of polymatrix games. It requires \emph{symmetry} in the
hyperedges: for all pair of players $i,j\in V, i \neq j, \{i,j\} \in
\mathcal{C}_i \cap \mathcal{C}_j$.} In contrast to
\emph{hypergraphical games}~\citep{Papadimitriou_2005}, a GMhG is more
expressive, in part because a GMhG does not
require that the same ``sub-game'' (i.e., local-clique payoff
hypermatrix) be shared among all players in the clique of the
``sub-game.'' For example, a local-clique payoff hypermatrix may
appear in the summation defining the local payoff hypermatrix of
exactly one player.~\footnote{Appendix~\ref{app:hyper} expands on the
  relation between GMhGs and hypergraphical games.}
A GMhG has the polynomial intersection property and thus a polynomial correlated equilibrium scheme~\cite{Papadimitriou_2005}.

\paragraph{Representation size.} The representation size of a GMhG is $\Theta(\sum_{i \in V} \sum_{C
  \in \mathcal{C}_i} \prod_{j \in C} |A_j|) = O(n \, l \, m^c)$, where
$l \equiv \max_{i \in V} |\mathcal{C}_i|$ and $c \equiv \max_{i \in V}
\max_{C \in \mathcal{C}_i} |C|$. Hence, the size is dominated by the
representation of the local-clique payoff matrices, which are each of
size exponential in their respective clique size. However, this
representation size could be considerably smaller than for a graphical
game, which is exponential in the neighborhood size. For example, if
for each $i$, we have $|\mathcal{C}_i| \leq k$, and for each $C \in
\mathcal{C}_i$, we have $|C| = 2$,
then the GMhG
becomes a \emph{graphical poly-matrix game}, with representation size
$O(n \, k \, m^2)$, \emph{linear} in the maximum number of neighbors $k$, compared to $O(n m^k)$ for a standard graphical
game with ``tabular'' representations, which is \emph{exponential} in
$k$.

\paragraph{Payoff Scale.} Normalizing the payoff of a GG in standard
local strategic/normal-form takes linear time in the representation
size of the game, because we can find the minimum and maximum local
payoff values for each local payoff hypermatrix (which, for each player, is exponential in
the size of the player's neighborhood) simply by going over each
payoff value in the hypermatrix in sequence. However, such an approach
is intractable in GMhGs in general. Denote the
maximum and minimum payoff values for each player $i
\in V, u_i \equiv \max_{x_{N_i}} \sum_{C \in \mathcal{C}_i} M'_{i,C}(x_C)$
and $l_i \equiv \min_{x_{N_i}} \sum_{C \in \mathcal{C}_i}
M'_{i,C}(x_C)$, respectively. Computing \emph{both} $u_i$ and $l_i$ is
NP-hard. To see this, first note that $u_i$ and $l_i$ are the result of a max
and min operation over an
additive function of the set of the player's hyper-edges and its
possible joint-actions. 
It is easy to reduce the problem of finding a solution to an
arbitrary contraint network to that of computing \emph{both}  $u_i$ and $l_i$ for each player
$i$. Hence, in general, we do not have much of a choice but to
assume that the payoffs of all players are in the same scale, so that
using a global approximation-quality value $\epsilon$ is meaningful; and to
compute the individual maximum and minimum values of each hypermatrix
payoff of each player as a way to set up the sparse uniform-discretization.

Some additional notation is necessary before stating the theorem. Denote by $u_{i,C} \equiv \max_{x_C \in A_C} M'_{i,C}(x_C)$ and $l_{i,C} \equiv \min_{x_C \in A_C} M'_{i,C}(x_C)$ the largest and smallest payoff values achieved by the local-grid payoff hypermatrix $M'_{i,C}$, respectively; and by $R_{i,C} \equiv u_{i,C} - l_{i,C}$ its {\em largest range\/} of values.
\begin{theorem}
{\bf (Sparse Nash-Equilibria Representation)}
\label{the:spdisc}
 For any graphical multi-hypermatrix game and any $\epsilon$ such that \[0  < \epsilon \leq 2 \,  \min_{i \in V} \, \frac{\sum_{C \in \mathcal{C}_i} R_{i,C} \, (|C| - 1)}{\max_{C' \in \mathcal{C}_i} |C'| - 1} \, ,\] a (uniform) discretization with \[ s_i = \left\lceil \frac{2 \, |A_i| \, \max_{j \in \Neigh_i} \sum_{C \in \mathcal{C}_j} R_{j,C} \, (|C| - 1) }{ \epsilon }\right\rceil \] for each player $i$ is sufficient to guarantee that for {\em every\/} Nash equilibrium of the game, its closest (in $\ell_\infty$ distance) joint mixed strategy in the induced discretized space 
is also an $\epsilon$-Nash equilibrium of the game.
\end{theorem}
The proof of the theorem appears in Appendix~\ref{app:proof}, \emph{deliberately}. This is
not only because it is simple, based on algebraic manipulations only,
but also because the proof is \emph{irrelevant} to the evaluation of the
\emph{significance} of the actual result stated in the
theorem. Stating the proof here would simply be distracting, getting
in the way of the really important objective: \emph{determining the importance of the
theorem on its own.}

The theorem establishes minimal sizes for each player's individually
uniform-discretization to capture a compact representation of
\emph{all} Nash equilibria with respect to the resulting grid over joint mixed
strategies in GMhGs. \emph{Hence, as a result of the theorem, we can represent
every exact Nash equilibria of any GMhG via a sparse multi-dimensional
grid over the joint mixed-strategies induced by a sparse
uniform-discretization of the individual probabilities of each action
of each mixed strategy of each player individually.} This sparse
discretization can lead not only to compact representations of the set
of exact NE, but also to tractable algorithms or effective methods
for the computation of approximate Nash equilibria in some classes of
games. We have already seen \emph{examples of important representational and
computational implications that result from the theorem's application} in the previous section,
which focused \emph{on GGs} and some of their specializations, such as \emph{normal-form games.}

Also, the generalized bound has already proved useful to derive and
analyze dynamic-programing algorithms for computing approximate Nash
equilibria in non-trivial special cases of \emph{interdependent defense
(IDD) games} with
specific, practical graph structures.~\citet{Chan12} recently
introduced IDD games to model, study, and analyze security and defense
mechanism for deterrence in network-structured
interdependent security systems under the risk of a deliberate attack
from an external agent. One objective is the
potential to study the effect of minimal interventions in overall
system security for deterrence purposes. 

Another example in which the generalized bounds have proven useful is
in the design of specific instantiations of CPR that lead to
versions of survey propagation for arbitrary GGs, which generalize
{\bf NashProp}~\citep{ortizandkearns03}, and for
linear influence games ~\citep{IrfanAAAI,IrfanTechReport,IrfanPhD}. In the particular
case of {\bf Survey NashProp}, the theorem helped us jump a major hurdle.  Each message-passing process in most general
CPR instances require time exponential in the size of the largest
variable domains in the CSP; this is in contrast to {\bf NashProp}
whose running time is linear in the domain size. In the case of {\bf Survey NashProp}, that would be
\emph{exponential} in the number of mixed-strategies it considers due to the
discretization of the mixed-strategy space; for {\bf NashProp}, on the
other hand, that would be quasi-linear. But using the generalized
version of the sparse discretization result given in the last theorem, we can now perform each message-passing in time exponential in $k^2$, which means \emph{quasi-linear} in the size of
the input with respect to $k$, while keeping the size of the messages
\emph{quadratic} in the representation size of the game, assuming $m$
is bounded by a constant, independent of $n$. In short, this means that we can now perform each
message-passing step in time \emph{quasi-linear in the size of the game}! 

Unlike for standard GGs, the broader algorithmic and
computational implications of the sparse discretization in GMhGs
remains wide open for the most part. The next section lists and
briefly discusses several of such open problems.

\section{Concluding Remarks and Open Problems to the AI Community}

While sparse-discretization result for \emph{graphical games (GGs)} and its algorithmic
implications are actually old, going back to a simple unpublished note
wrote in 2002~\citep{Ortiz_2002}, posted online a few months later
as part of a course
on computational game theory at the University of Pennsylvania,~\footnote{\url{http://www.cis.upenn.edu/\homedir mkearns/teaching/cgt/}}
it has gained
considerable relevance over the last few years because of the need to
deal with even more compact representations than GGs provide, and algorithms/heuristics that require running times
exponential in the size of the discretization, thus 
linear in the size of the model, under many reasonable conditions. Of
particular relevance along this line is prior and current exploration of
extensions, based on constraint propagation
relaxation~\citep{ortiz08}, of survey propagation to the problem of computing
approximate Nash equilibria in GGs and recent work on 
linear influence games, and their
generalizations~\citep{IrfanAAAI,IrfanTechReport,IrfanPhD}. As
stated at the end of Section~\ref{sec:motiv}, the new bounds may also provide improvements to previous
discretization-based schemes for computing $\epsilon$-Nash equilibria
in other similar models, such as those in the work
of~\citet{Singh_et_al_2004} and~\citet{Soni_et_al_2007}.

Several open questions remain, perhaps most important of which are
those related to computation in GG generalizations such as GMhGs. The
following is a partial list of open problems for which 
the AI community, and in particular, the community on constraint
networks~\citep{Dechter_2003} for CSPs, 
are very well
equipped to solve.
\begin{itemize}


\item Can we combine the sparse-support and sparse-discretization
  approach to obtain a better algorithm? For example, can we reduce dependency on the
  number of actions or the accuracy parameter, without unduly
  increasing the dependency on either? A naive approach does
  not seem to work, but maybe a more sophisticated approach would. Are
  the sparse-discretization and sparse-support approaches simply incompatible?

\item Can we do away with the \emph{uniform} discretization scheme and still
  guarantee that every Nash equilibria is near a grid point? Would it
  help if the objective is to find a single approximate solution only? What if we want our
  approximate Nash equilibria to also be ``stable'' (i.e., near an
  actual Nash equilibrium)?

\item Computing \emph{exact} Nash equilibria, or more formally, 
  approximate Nash equilibria with exponentially small approximation
  parameter, in $n$-player $2$-action 
  poly-matrix games, which is a GG generalization, is
  PPAD-complete. This result is an implicit corollary of the work
  of~\citet{daskalakis:195}, although they do not provide any
  explicit, formal proof.  Yet, the design of efficient algorithms or effective heuristics for computing
  \emph{approximate} Nash equilibria in such GG generalizations is wide open! A naive application of the
  sparse-discretization approach does not seem to help. Can sparse
  discretization or a sparse-support approach help? If so, how?

\item There is no substantive work on using the sparse-support
  approach for GGs, at least not directly. \emph{Action-graph
    games}~\citep{AGG-GEB} contain GGs.~\citet{ThompsonSEMinAGs}
  designed methods based on the
sparse-support approach to compute approximate Nash equilibria in the context of action-graph games. How exactly do these methods
apply to GGs directly, or indirectly, if at all? Can we adapt those
methods to obtain direct versions, still based on sparse support, for computing approximate Nash
equilibria in GGs? Is the sparse-support approach simply incompatible
with GGs, when it comes to the design of efficient or effective
algorithms for computing approximate Nash equilibria that are also
natural within the GG context?

\item One can view a sparse discretization as generating an equivalent game
  in which the discretized mixed-strategy space of the original game becomes the pure
  strategies of the equivalent game. It seems possible
  to extend the computational results
  of~\citet{GottlobHardEasy} by considering the computation of
  pure-strategy Nash equilibria of the equivalent game, which corresponds to
  approximate mixed-strategy Nash equilibria of the original
  game. How to do this exactly requires further research and thus
  remains open.
\end{itemize}

\section*{Acknowledgement}

This manuscript was supported in part by NSF CAREER Award IIS-1054541.

\appendix 

\section{What is Beyond the Scope of This Short Paper?}
\label{app:scope}

This section briefly discusses some motivation for the study of Nash
equilibria and other technical
aspects that, while interesting, addressing them within this short
paper will 
make it lose its focus.

\subsection{Significance of Non-cooperative Game Theory
  and Nash Equilibria}
\label{app:NE}

At this point, it would not be an exaggeration to say that game theory
and the concept of Nah equilibrium have touched almost all areas of
science and engineering. The list of the multitude of areas and
applications would be too long to present here. It is beyond the
scope of this note to discuss the theoretical and practical relevance
that non-cooperative game theory and the concept of a Nash equilibrium
have to real-world
problem. But, here are a few quotes from the organization that awards
what most people call the ``Nobel Prize in Economics'' every year.
They give a taste for the broader impact of non-cooperative game
theory and the Nash equilibria as a solution concept.
\begin{quote}
\emph{Summary:} ``The Sveriges Riksbank Prize in Economic Sciences in Memory of Alfred
Nobel 1994 was awarded jointly to John C. Harsanyi, John F. Nash
Jr. and Reinhard Selten {\em "for their pioneering analysis of equilibria in the theory of non-cooperative games".''}~\footnote{\url{http://www.nobelprize.org/nobel_prizes/economic-sciences/laureates/1994/}}
\end{quote}
\begin{quote}
\emph{Press Release:} ``Game theory is a mathematical method for analyzing strategic interaction.''~\footnote{\url{http://www.nobelprize.org/nobel_prizes/economic-sciences/laureates/1994/press.html}}
\end{quote}
\begin{quote}
\emph{Press Release:} ``Many interesting economic issues, such as the analysis of oligopoly, originate in non-cooperative games. In general, firms cannot enter into binding contracts regarding restrictive trade practices because such agreements are contrary to trade legislation. Correspondingly, the interaction among a government, special interest groups and the general public concerning, for instance, the design of tax policy is regarded as a non-cooperative game. Nash equilibrium has become a standard tool in almost all areas of economic theory. The most obvious is perhaps the study of competition between firms in the theory of industrial organization. But the concept has also been used in macroeconomic theory for economic policy, environmental and resource economics, foreign trade theory, the economics of information, etc. in order to improve our understanding of complex strategic interactions.''~\footnote{\url{http://www.nobelprize.org/nobel_prizes/economic-sciences/laureates/1994/press.html}}
\end{quote}

\subsection{Other Game-theoretic Graphical Models}
\label{app:GMs}

This note does not consider other kinds of graphical models for game theory such as
\emph{multiagent influence diagrams (MAIDs)}~\citep{koller03}, which
provide graphical models for extensive-form games, \emph{action-graph
  games}~\citep{AGG-GEB},~\footnote{\citet{ThompsonSEMinAGs} considers enumeration methods based on
sparse support to compute approximate Nash equilibria in the context of action-graph games.}  which exploit
``context-sensitive'' conditional expected-payoff independence, and
\emph{expected utility networks (EUNs)}~\citep{lamura00}, which
focuses on models for the ``strategic'' decision-making aspects of a single, individual player.

\subsection{Relative and Constant Approximations}

\label{app:rel}

This note only considers \emph{absolute
  approximations}, the most commonly
studied form of approximation of Nash equilibria. Relative approximations have also been
the subject of study, but mostly for non-graphical
models~\citep{HemonANE}.

A very recent interest is to provide polynomial-time
algorithms for computing Nash equilibria of \emph{specific, constant}
approximation quality (see, e.g., work by~\citeR{TsaknakisOptANE,Daskalakis_et_al_2009,HemonANE,BosseNew} and the
references therein). Most results of this kind are in $2$-player
games.
Several authors have shown how to turn such
polynomial algorithms for a constant approximation quality from
$2$-player to results for $n$-player, with a \emph{larger, still constant}
approximation quality. Most of that work uses a sparse-support
approach.
This type of
approximate Nash equilibria problem does not seem to have been studied for
arbitrary GGs, just specializations such as normal-form
games. This paper
does not consider such problems to compute approximate Nash equilibria with constant
approximation quality here, except for revisiting it in the list of
open problems at the end.

\section{Expanding on the Relation between GMhGs and Hypergraphical
  Games}
\label{app:hyper}

In the standard
  definition of hypergraphical games we have a hypergraph
  $(V,\mathcal{E})$, where each vertex $i \in V$ corresponds to a
  player $i$ in the game and $\mathcal{E} \subset 2^V$ is a set of
  hyperedges (i.e., sets of subsets of $V$). Each player $i \in
  V$ has a finite set of actions $A_i$. There is ``game'' for each hyperedge 
  $C \in \mathcal{E}$. By ``game'' here we mean that each player $i
  \in C$ has the same set of
  actions $A_i$ in all (local) ``games'' defined by $\mathcal{E}$, and a local
  payoff hypermatrix $M''_{i,C}(x_C)$ that is a function of the
  joint-actions $x_C \in A_C$ of all the players in $C$. Let
  $\mathcal{C}''_i \equiv \{ C \in \mathcal{E} \mid i \in C \}$ be the
 (local) set of hyperedges that the hypergraphical game induces for
 each player $i \in V$. The final/global
  payoff functions of each player $i$ in the hypergraphical game is $M_i(x)
  \equiv \sum_{C \in \mathcal{C}''_i} M''_{i,C}(x_C)$. 

Note that by
  definition, for all pairs of players $i,j\in V$, we have that, for any
  hyperedge $C \in \mathcal{E}$ such that $i,j\in C$, the following symmetry property must
  hold in a hypergraphical game: $C \in \mathcal{C}''_i$ if and only if $C \in
  \mathcal{C}''_j$. \emph{GMhGs do not require this symmetry
    condition.} 

Of course, from a \emph{mathematical}
  perspective, one can always
  take a GMhG with a set of player's hyperedges $\mathcal{C}_i$ and
  local hypermatrices $M'_{i,C}$ for
  each player $i \in V$ and hyperedge $C\in \mathcal{C}_i$, and turn it into a hypermatrix game; just as one can
  take a GG and turn it into a normal-form game. In the
  case of the GMhG transformation, we can set
  the set of hyperedges $\mathcal{E}$ of the GMhG-induced hypergraphical
  game to $\mathcal{E} \equiv \cup_{i\in V} \cup_{C \in \mathcal{C}_i}
  C$, and, letting the GMhG-induced hyperedges $\mathcal{C}'_i \equiv \{ C \in \mathcal{E} \mid i
  \in C\}$ for all $i \in V$. Then, for every $i \in V$, we must
  either create a local \emph{zero-valued}
  hypermatrix $M''_{i,C}(x_C) \equiv 0$, for all $x_C \in A_C$ if $C
  \not\in \mathcal{C}_i$; otherwise, if $C \in \mathcal{C}_i$, set the local hypermatrix
  of the hypergraphical game to $M''_{i,C} \equiv M'_{i,C}$.

  Indeed, one can
take any game with a finite number of players and actions and always
turn it back into a normal-form game. But, from a \emph{computational}
perspective, such transformations could take an exponential amount of
time and space, defeating the main purpose for introducing
compact,
tractable, and flexible
representations of significant expressive power in the first place!
Although, from a \emph{theoretical} standpoint, the given transformation from GMhGs to hypergraphical games
is computationally efficient, from a \emph{practical} standpoint,
the loss in expressive power is clear, and significant.

\section{Proof of Theorem~\ref{the:spdisc}}
\label{app:proof}


To simplify notation, given any joint mixed-strategy (i.e., a product
distribution) $p$, for all $B \subset V$, and $x_B \in A_B$, we denote by $p(x_B)
\equiv \prod_{i \in B} p_i(x_i) = \sum_{x_{-B}} p(x_B,x_{-B})$ the joint mixed-strategy over players
in $B$ only (i.e., marginal product-distributions of $p$ over the
joint-actions of players in $B$). Let $p$ and $q$ be two joint mixed strategies and, for each player $i$ and each action $x_i$, denote by $\Delta_i(x_i) \equiv p_i(x_i) - q_i(x_i)$. In a slight abuse of notation, let $\Delta(x_S) \equiv \prod_{k \in S} \Delta_k(x_k)$. 

The following very simple lemma is a cornerstone of the proof.
\begin{lemma}{\bf(Product-Distribution Differences)}
\label{lem:diff}
For any clique $B \subset V$ of players, for any clique joint action $x_{B}$, 
\[
p(x_B) - q(x_B)
= \sum_{S \in 2^B -\emptyset} \Delta(x_S) \, q(x_{B-S}) \; .
\]
\end{lemma}
\begin{proof}
The lemma follows by applying a binomial expansion:
\begin{align*}
p(x_B) =& \prod_{j \in B} \left( q_j(x_j) + \Delta_j(x_j) \right)\\ 
=& \sum_{S \in 2^B} \Delta(x_S) \, q(x_{B-S}) \\
=& q(x_B) + \sum_{S \in 2^B -\emptyset} \Delta(x_S) \,
q(x_{B-S}) \; . 
\end{align*}
\end{proof}

To further simplify the presentation of the proof it is convenient to
introduce a slight abuse of notation: for all, $i \in C, C \in
\mathcal{C}_i, B,S \subset C, B \cap S = \emptyset, x_S \in A_S, p_{C-B-S} \in \mathcal{P}_{C-B-S}$, let
$M'_{i,C}(x_S,\Delta_B,p_{C-B-S}) \equiv \sum_{x_B \in A_B} \Delta(x_B) M'_{i,C}(x_S,x_B,p_{C-B-S})$.

The following useful claim follows immediately from the last lemma of joint
product distribution
differences (Lemma~\ref{lem:diff}).
\begin{claim}
Under the conditions of Lemma~\ref{lem:diff},
for all $i \in V$, $C \in
\mathcal{C}_i$, $B \subset C$, $x_B \in A_B$ and $p_{C-B}, q_{C-B} \in
\mathcal{P}_{C-B}$, we have
\begin{align*}
M'_{i,C}(x_B,p_{C-B}) - M'_{i,C}(x_B,q_{C-B}) = \sum_{S \in 2^{C-B} -\emptyset} 
   M'_{i,C}(x_B,\Delta_S,q_{C-B-S})
 \; .
\end{align*}
\end{claim}
\begin{proof} 
Applying the last lemma on the differences between product distributions (Lemma~\ref{lem:diff}), we have 
\begin{align*}
M'_{i,C}(x_B,p_{C-B}) - M'_{i,C}(x_B,q_{C-B}) = &
 \sum_{x_{C-B}} \left[ \sum_{S \in 2^{C-B} -\emptyset} \Delta(x_S)
   q(x_{C-B-S})  \right] M'_{i,C}(x_C) \\
  =& \sum_{S \in 2^{C-B} -\emptyset} \sum_{x_{S}} \Delta(x_S)
   \sum_{x_{C-B-S}} q(x_{C-B-S}) M'_{i,C}(x_C)
 \\
= & \sum_{S \in 2^{C-B} -\emptyset} 
   M'_{i,C}(x_B,\Delta_S,q_{C-B-S})
 \; .
\end{align*}
\end{proof}
Using some algebra we can show another useful claim.
\begin{claim}
Under the conditions of Lemma~\ref{lem:diff}, 
for all $i \in V$, and $C \in \mathcal{C}_i$, we have 
\begin{align*}
\sum_{S \in 2^C - \emptyset}
M'_{i,C}(\Delta_S,q_{C-S}) = \sum_{B \in 2^{C-\{i\}} -
    \emptyset} M'_{i,C}(p_i,
  \Delta_B,q_{C-B-\{i\}}) \; .
\end{align*}
\end{claim}
\begin{proof}
First note that we can decompose the left-hand side of the equation in
the claim  as
\begin{align*}
\sum_{S \in 2^C - \emptyset}
M'_{i,C}(\Delta_S,q_{C-S}) = &
\sum_{B \in 2^{C-\{i\}} -
    \emptyset}
M'_{i,C}(q_i, \Delta_B,q_{C-B-\{i\}}) + \\
& 
\sum_{B \in 2^{C-\{i\}} -
    \emptyset}
M'_{i,C}(\Delta_i,\Delta_B,q_{C-B-\{i\}}) \; .
\end{align*}
Now note that, using the definition of $\Delta_i$, we have
\begin{align*}
M'_{i,C}(\Delta_i,\Delta_B,q_{C-B-\{i\}}) = & M'_{i,C}(p_i,\Delta_B,q_{C-B-\{i\}}) -
M'_{i,C}(q_i,\Delta_B,q_{C-B-\{i\}}) \; .
\end{align*}
The claim follows after making the appropriate substitution for that expressions and
simplifying:
\begin{align*}
\sum_{S \in 2^C - \emptyset} M'_{i,C}(\Delta_S,q_{C-S})
= &
\sum_{B \in 2^{C-\{i\}} -
    \emptyset}
M'_{i,C}(q_i, \Delta_B,q_{C-B-\{i\}}) + \\
& \sum_{B \in 2^{C-\{i\}} -
    \emptyset}
(M'_{i,C}(p_i,\Delta_B,q_{C-B-\{i\}}) -
M'_{i,C}(q_i,\Delta_B,q_{C-B-\{i\}})) \\
= & \sum_{B \in 2^{C-\{i\}} -
    \emptyset}
M'_{i,C}(p_i,\Delta_B,q_{C-B-\{i\}})
\; .
\end{align*}
\end{proof}
Suppose $p$ is a Nash equilibrium of the game, which must exist by Nash's
Theorem~\citep{nash51}.
Applying the last two claims above,
we obtain
\begin{align*}
M_i(p) = & M_i(q) + \sum_{C \in \mathcal{C}_i} \sum_{S \in 2^C - \emptyset}
M'_{i,C}(\Delta_S,q_{C-S})\\
= & M_i(q) + \sum_{C \in \mathcal{C}_i} \sum_{B \in 2^{C-\{i\}} -
    \emptyset}
  M'_{i,C}(p_i, \Delta_B,q_{C-B-\{i\}})  \\
\geq & \max_{x'_i} M_i(x'_i, p_{-i})\\
= & \max_{x'_i} \left[ M_i(x'_i, q_{-i}) + 
\sum_{C \in \mathcal{C}_i} \sum_{B \in 2^{C -\{i\}} - \emptyset} M'_{i,C}(x'_i,\Delta_B,q_{C-B-\{i\}}) \right] .
\end{align*}
Rearranging and simplifying,
we obtain the following expression: 
\begin{align*}
\nonumber
M_i(q) \geq & 
\max_{x'_i} M_i(x'_i, q_{-i}) + \\
& 
\sum_{C \in \mathcal{C}_i} \sum_{B \in 2^{C-\{i\}} - \emptyset}
\left( M'_{i,C}(x'_i,\Delta_B,q_{C-B-\{i\}}) - M'_{i,C}(p_i,\Delta_B,q_{C-B-\{i\}}) \right)
\; .
\end{align*}

Let $q$ be the closest (in $\ell_\infty$ distance) joint mixed
strategy  in $\widetilde{I}$, defined using sizes $s_i$ as given in the
statement of the theorem, to exact Nash equilibrium $p$. Hence, we have 
\begin{align*}
|\Delta_i(x_i)| \leq
\frac{\epsilon}{2 \, |A_i| \, \max_{j \in \Neigh_i} \sum_{C \in
  \mathcal{C}_j}\, R_{j,C} \, (|C| - 1)} \; .
\end{align*}
Consider the conditional expected hypermatrix payoff
difference in parenthesis within the innermost summation inside the maximization 
of the equilibrium condition above. 
Noting that \[ |M'_{i,C}(x'_i,x_B,q_{C_{-i,B}}) -
M'_{i,C}(p_i,x_B,q_{C_{-i,B}})| \leq R_{i,C} , \] we obtain the
following lower bound on that innermost summation:
\begin{align}
 \nonumber \left( M'_{i,C}(x'_i,\Delta_B,q_{C_{-i,B}}) - M'_{i,C}(p_i,\Delta_B,q_{C_{-i,B}}) \right) 
\geq & \sum_{x_B} \left[ \prod_{k \in B} | \Delta_k(x_k) | \right] (- R_{i,C})\\
\label{eqn:inmost_sum} = & - R_{i,C} \sum_{x_B} \prod_{k \in B} | \Delta_k(x_k) | \, .
\end{align}
We can upper bound the last factor in the right-hand side of the last expression as
\begin{eqnarray*}
\sum_{x_B} \prod_{k \in B} | \Delta_k(x_k) | & = &  \prod_{k \in B} \sum_{x_k} | \Delta_k(x_k) |\\
& \leq & \prod_{k \in B} \sum_{x_k} \frac{\epsilon}{2 \, |A_k| \, \max_{j \in \Neigh_k} \sum_{C' \in \mathcal{C}_j}\, R_{j,C'} \, (|C'| - 1)} \\
&= & 
\prod_{k \in B} \frac{\epsilon}{2 \, \max_{j \in \Neigh_k} \sum_{C' \in \mathcal{C}_j}\, R_{j,C'} \, (|C'| - 1)} \\
&\leq& 
\prod_{k \in B} \frac{\epsilon}{2 \, \sum_{C' \in \mathcal{C}_i}\, R_{i,C'} \, (|C'| - 1)} \\
&=& 
\left( \frac{\epsilon}{2 \, \sum_{C' \in \mathcal{C}_i}\, R_{i,C'} \, (|C'| - 1)} \right)^{|B|} \, .
\end{eqnarray*}
Thus, using the resulting lower bound on the expression given in~(\ref{eqn:inmost_sum}), we obtain
\begin{align*}
M_i(q) \geq
\max_{x'_i} M_i(x'_i, q_{-i}) - 
\sum_{C \in \mathcal{C}_i} \sum_{B \in 2^{C - \{i\}} - \emptyset} R_{i,C} \left( \frac{\epsilon}{2 \, \sum_{C' \in \mathcal{C}_i}\, R_{i,C'} \, (|C'| - 1)} \right)^{|B|} \, .
\end{align*}
The second term in the right-hand side of the last expression equals
\begin{align*}
& \sum_{C \in \mathcal{C}_i} R_{i,C} \sum_{B \in 2^{C-\{i\}} - \emptyset} \left( \frac{\epsilon}{2 \, \sum_{C' \in \mathcal{C}_i}\, R_{i,C'} \, (|C'| - 1)} \right)^{|B|}\\
&= \sum_{C \in \mathcal{C}_i} R_{i,C} \left[ \left( 1 + \frac{\epsilon}{2 \, \sum_{C' \in \mathcal{C}_i}\, R_{i,C'} \, (|C'| - 1)} \right)^{|C|-1} - 1 \right] 
\\
&\leq \sum_{C \in \mathcal{C}_i} R_{i,C} \left[ \frac{\epsilon \; (|C|-1)}{\sum_{C' \in \mathcal{C}_i}\, R_{i,C'} \, (|C'| - 1)} \right] = \epsilon \, .
\end{align*}
The last inequality follows from using the upper bound condition on $\epsilon$ given in the statement of the theorem and applying the well-known inequality, $1 + z \leq \exp(z) \leq 1 + z + z^2$ for $|z| < 1$~\cite{Cormen_et_al_1990}.
This completes the proof of the theorem. 

\section{Approximation Schemes}
\label{app:approx}

This appendix presents a brief, general description of the different
approximation schemes. To begin, it is important to note that
it is most common to define approximations in terms of \emph{relative}
factors. Here, given the standard/traditional mathematical definition of
approximate Nash equilibria in a game setting, the description assumes that the approximation quality
is in \emph{absolute} terms. The description assumes that the problem
is that of computing absolute
approximations of Nash equilibria in games, as defined in the main body of the text.
\begin{definition}
Let $N$ be the representation size of the input game $\mathcal{G}$ and $\epsilon > 0$ the
(absolute) approximation quality. An algorithm {\sc A} is a
\emph{polynomial-time approximation scheme (PTAS)} if, for every
$\epsilon > 0$, on input $\mathcal{G}$, {\sc A}
runs in time polynomial in $N$. If, in addition, {\sc A} runs in time
polynomial in $\frac{1}{\epsilon}$, then {\sc A} is a
\emph{fully-polynomial-time approximation scheme (FPTAS)}. If, instead,
{\sc A} runs in time $N^{O(\mathit{polylog}(N))}$, where
$\mathit{polylog}(N)$ denotes a polynomial function of $\log{N}$, then
{\sc A} is a
\emph{quasi-polynomial-time approximation scheme (quasi-PTAS)}.
\end{definition}

\bibliographystyle{theapa}
\bibliography{games}

\end{document}